\documentclass[letterpaper, 10 pt, conference]{ieeeconf}

\IEEEoverridecommandlockouts       
\makeatletter
\let\NAT@parse\undefined
\makeatother

\usepackage{hyperref}

\usepackage{amsmath} 
\usepackage{amssymb} 
\usepackage{graphicx, xcolor}
\usepackage{bm}
\usepackage{balance}
\usepackage{cite}
\usepackage{tabularx}
\usepackage{booktabs}
\usepackage{multirow}
\usepackage{siunitx}
\usepackage[font=footnotesize]{caption}

\usepackage[utf8]{inputenc}
\usepackage{pgfplots}
\DeclareUnicodeCharacter{2212}{−}
\usepgfplotslibrary{groupplots,dateplot}
\usetikzlibrary{patterns,shapes.arrows}
\pgfplotsset{compat=newest}

\hypersetup{
  colorlinks    = true, 
  urlcolor      = blue, 
  linkcolor     = blue, 
  citecolor     = red   
}
\newcommand{\vect}[1]{\textbf{#1}}

\newcommand{\RR}{\mathbb{R}}
\newcommand{\T}{^{\intercal}}

\newtheorem{theorem}{Theorem}

\newtheorem{definition}{Definition}

\newtheorem{assumption}{Assumption}

\newcommand{\textref}[2]{\hyperref[#1]{#2~\ref*{#1}}}

\title{\LARGE \bf
Multi-Step Model Predictive Safety Filters: Reducing Chattering by Increasing the Prediction Horizon}

\author{Federico Pizarro Bejarano, Lukas Brunke, and Angela P. Schoellig
\thanks{The authors are with the Learning Systems and Robotics Lab (\url{www.learnsyslab.org}), University of Toronto, Canada, and affiliated with the University of Toronto Robotics Institute and the Vector Institute for Artificial Intelligence in Toronto. Lukas Brunke and Angela P. Schoellig are also with the Technical University of Munich and the Munich Institute for Robotics and Machine Intelligence~(MIRMI), Germany. E-mails: \{federico.pizarrobejarano, lukas.brunke, angela.schoellig\}@robotics.utias.utoronto.ca}}

\begin{document}

\maketitle
\thispagestyle{empty}
\pagestyle{empty}

\begin{abstract}
Learning-based controllers have demonstrated superior performance compared to classical controllers in various tasks. However, providing safety guarantees is not trivial. Safety, the satisfaction of state and input constraints, can be guaranteed by augmenting the learned control policy with a safety filter. Model predictive safety filters~(MPSFs) are a common safety filtering approach based on model predictive control~(MPC). MPSFs seek to guarantee safety while minimizing the difference between the proposed and applied inputs in the immediate next time step. This limited foresight can lead to jerky motions and undesired oscillations close to constraint boundaries, known as chattering. In this paper, we reduce chattering by considering input corrections over a longer horizon. Under the assumption of bounded model uncertainties, we prove recursive feasibility using techniques from robust MPC. We verified the proposed approach in both extensive simulation and quadrotor experiments. In experiments with a Crazyflie 2.0 drone, we show that, in addition to preserving the desired safety guarantees, the proposed MPSF reduces chattering by more than a factor of 4 compared to previous MPSF formulations. 
\end{abstract}

\section{Introduction}
Robotic systems are increasingly deployed to perform tasks in complex environments, including autonomous driving~\cite{zeus}, aerial delivery~\cite{drone_delivery}, and surgery~\cite{surgery}. These tasks often suffer from significant uncertainties in the system dynamics, rendering classical controllers less effective~\cite{brunke_safe_2021}. Learning-based controllers have demonstrated improved performance on complex tasks~\cite{games, silver-nature-2016, openai-2019}.
However, reinforcement learning controllers typically lack safety guarantees and rarely enforce state constraints such as road lane boundaries for a self-driving car~\cite{brunke_safe_2021}. This prevents reinforcement learning controllers from being deployed on safety-critical systems despite promising results and demonstrations. Combinations of machine learning and model-based control~(notably learning-based MPC) have become popular methods of leveraging the benefits of both approaches. However, these approaches do not have the adaptability across diverse systems and tasks of model-free reinforcement learning~\cite{brunke_safe_2021}.

Safety filters allow arbitrary controllers to be implemented with safety guarantees, including deep learning controllers~\cite{brunke_safe_2021}. Model predictive safety filters~(MPSFs)~\cite{zeilinger_linear} are a category of safety filters that leverage model predictive control~(MPC) to predict whether uncertified~(i.e., potentially unsafe) inputs sent from the controller will violate the constraints. In the case of a potential future constraint violation, the MPSF determines the minimal deviation from the uncertified input that results in constraint satisfaction~(see \textref{fig:general_sf_model}{Fig.}). 

\begin{figure}[t]
  \centering
  \includegraphics[width=1.0\linewidth]{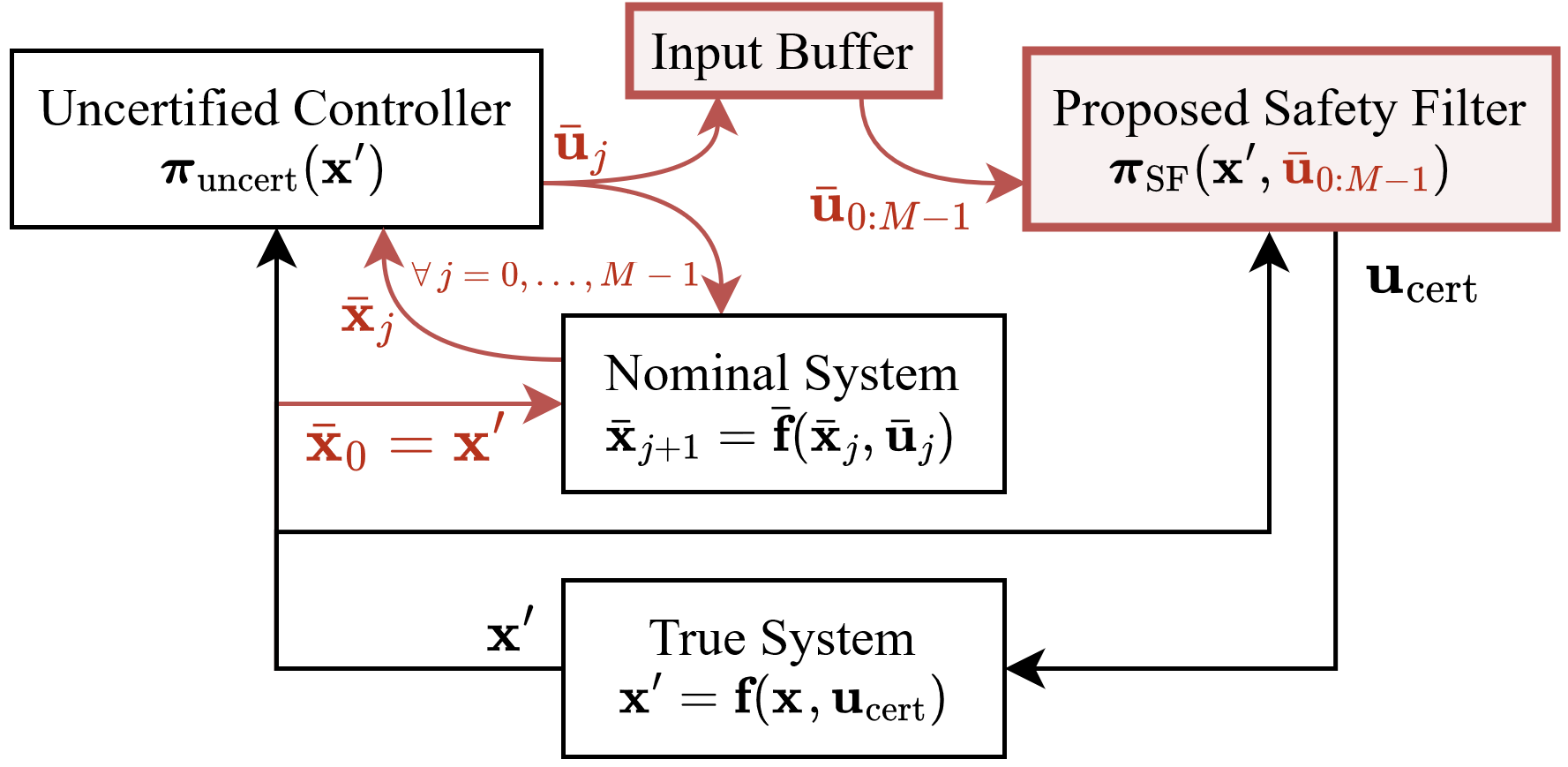}\\
  \caption{Our proposed safety filter approach~$\boldsymbol{\pi}_{\text{SF}}$ for safely controlling the system~$\vect{x}^{\prime} = \vect{f}(\vect{x}, \vect{u})$, with the novel components highlighted in red. The safety filter simulates an input trajectory for the next~$M$ time steps, using the nominal system~$\bar{\vect{f}}$ and the uncertified controller~$\boldsymbol{\pi}_{\text{uncert}}(\vect{x}^{\prime})$. The input  trajectory~$\bar{\vect{u}}_{0:M-1}$ along with the state~$\vect{x}^{\prime}$ is passed to the safety filter~$\boldsymbol{\pi}_{\text{SF}}$, which outputs a safe control input~$\vect{u}_{\text{cert}}$. 
  }
  \label{fig:general_sf_model}
  \vspace*{-1mm}
\end{figure}

\begin{figure}[t]
  \centering
  \includegraphics[width=1.0\linewidth]{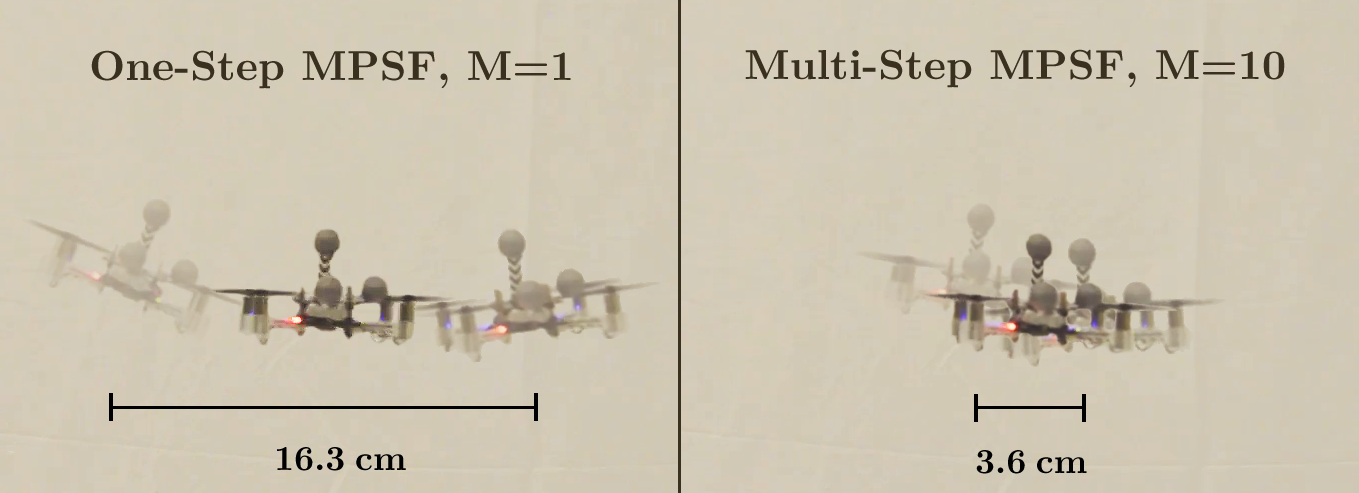}\\
  \caption{Experimental results on a Crazyflie 2.0 drone illustrating the chattering caused by the standard one-step MPSF versus the proposed multi-step MPSF with a filtering horizon of $M=10$. The proposed filter reduces the peak-to-peak amplitude of chattering near the constraint boundary from an average of \SI{16.3}{\centi\meter} to \SI{3.6}{\centi\meter} (78\% reduction).}
  \label{fig:drone_chat}
  \vspace*{-3mm}
\end{figure}

Despite strong theoretical guarantees inherited from the underlying MPC \cite{zeilinger_linear}, MPSFs may cause chattering and high-magnitude corrections to the uncertified control inputs~(see \textref{fig:drone_chat}{Fig.}). Chattering occurs when the controller directs the system towards a constraint boundary and is repeatedly stopped by the safety filter. This leads to an oscillatory behavior between the controller and the safety filter~\cite{Koller_LMPC}. Additionally, due to the MPSF objective only looking ahead one time step, the filter may allow a system to approach the boundary very closely before using a high-magnitude correction to keep the system inside the state constraints. Both of these behaviors can cause the system to leave the operation regime as the high-frequency input changes may not be accounted for in the system identification, increasing the model mismatch between the actual system and the nominal model in the MPSF~\cite{brunke_safe_2021}, potentially leading to constraint violations.

\textit{Contributions}: First, we propose generalizing the standard MPSF objective function to reduce chattering and other potentially unsafe corrective actions. Our approach inherits the theoretical recursive feasibility guarantees of the underlying MPC. Second, we propose metrics that reflect the magnitude of corrections and chattering. Previously, chattering has only been qualitatively described~\cite{Koller_LMPC, zeilinger_linear, reachability_w_GP}. The efficacy of the proposed MPSF is demonstrated in simulation using the \texttt{safe-control-gym}~\cite{safe-control-gym} cartpole and quadrotor systems, and on real-world Crazyflie 2.0 quadrotors (see \textref{fig:drone_chat}{Fig.}). As far as the authors know, this is only the second paper to implement MPSFs on real robots, where~\cite{racing} designed an MPSF specific for miniature remote-controlled car racing. Our code can be found at \url{https://tinyurl.com/mpsf-code}, and a video of our experiment can be found at \url{https://tinyurl.com/mpsf-video}.

\section{Related Work}
As learning-based controllers have grown in popularity, leveraging their strengths while also guaranteeing constraint satisfaction has become increasingly important. Many approaches for incorporating safety into model-free learning controllers serve to encourage safe behavior, reducing the probability of constraint violations. However, they do not provide hard guarantees on constraint satisfaction and are not generalizable to all learning controllers~\cite{brunke_safe_2021, ostafew}. 

Safety filters are necessary to guarantee the safety of an arbitrary controller. This is done by computing the minimal deviation from the latest uncertified control input that results in constraint satisfaction. 

Control barrier functions (CBFs) are frequently used in safety filtering frameworks~\cite{brunke_safe_2021}. The CBFs define safe sets as their super-level set~\cite{cbfs} and solve an optimization problem for each proposed control input to certify whether it is safe, and modify it otherwise. CBFs are usually defined for continuous-time systems~\cite{brunke_safe_2021}, which poses challenges when running discrete-time controllers. Finding CBFs can be difficult and requires a model of the system dynamics or offline data; analytical solutions can typically only be found for special cases~\cite{cbfs}. However, various methods for learning the CBF condition have recently been proposed~\cite{expert_cbf, lukas_bblr, black_box_cbf}. Another approach to certification using a safe set determines the set using Hamilton-Jacobi reachability analysis~\cite{HJRA} and switches between the original controller and a safe controller to keep the system inside the safe set~\cite{reachability_w_GP}. The CBF- and reachability-based safety filters compute safe inputs quickly and have strong theoretical guarantees, but determining an appropriate safe set is challenging, and accounting for constraints that change during operation is typically computationally infeasible~\cite{brunke_safe_2021}. 

Model predictive safety filters (MPSFs) 
determine certified control inputs
using model predictive control~\cite{zeilinger_linear}. At each iteration, the MPSF receives a potentially unsafe input and determines the state of the system after this input is applied using a model of the system dynamics. Then, from that future state, the MPSF uses the nominal model of the system to generate a trajectory to a safe terminal set in a number of steps up to the horizon of the MPC. This trajectory must be safe despite uncertainties, which is guaranteed using a known bound on the mismatch between the true system dynamics and the nominal model. If there is no feasible trajectory that will take the system from this future state to the terminal set, it will modify the original input so that this is possible.

MPSFs reduce offline computation and complex analytical work necessary for determining a safe set in favor of increased online computation~\cite{zeilinger_linear}. Since MPSFs are based on well-established MPC frameworks, existing optimization techniques allow MPSFs to be executed in real-time~\cite{boyd_efficient_mpc, qp_nl_mpc}. This makes them more suitable for complex systems in high dimensions, for which determining safety for all initial conditions
is computationally infeasible, but for which safety can be guaranteed for a specific initial condition. 

MPSFs, as well as other safety filters such as CBFs and reachability-based filters, can cause chattering~\cite{Koller_LMPC, zeilinger_linear}. This has only been partially addressed for MPSFs by penalizing the rate of change of the inputs~\cite{racing}, a common MPC approach for achieving smooth state and input trajectories~\cite{mpc_textbook}. However, this approach does not directly penalize corrections but reduces the rate of change of the inputs over the entire trajectory, even during safe actions, and thus is not generally suitable. Reducing chattering is the main goal of this paper. 


\section{Problem Formulation}
We consider a discrete, time-invariant system
\begin{align}
\label{eq:true-system}
    \vect{x}_{k+1} = \vect{f}(\vect{x}_{k}, \vect{u}_k),
\end{align}
where $\vect{x}_k \in \RR^n$ is the state at time step $k$, $\vect{u}_k \in \RR^m$ is the input, 
and $\vect{f}$ describes the dynamic behavior. The system is subject to known state and input constraints, $\vect{x} \in \mathbb{X}_{\text{c}}$ and 
$\vect{u} \in \mathbb{U}_{\text{c}}$, where $\mathbb{X}_{\text{c}} \subset \RR^n$ is closed and $\mathbb{U}_{\text{c}} \subset \RR^m$ is compact.

We assume that we only have access to a nominal model $\bar{\vect{f}}$ and uncertainty set $\mathbb{W}$ such that 
\begin{align}
\label{eq:nominal-system}
    \vect{x}_{k+1} = \bar{\vect{f}}(\vect{x}_{k}, \vect{u}_k) + \vect{w}(\vect{x}_k, \vect{u}_k) \,,
\end{align}
where $\vect{w}(\vect{x}, \vect{u}) \in \mathbb{W} \subset \RR^n \,, \forall \vect{x} \in \mathbb{X}_{\text{c}}\,, \forall \vect{u} \in \mathbb{U}_{\text{c}}$, and $\mathbb{W}$ is compact.  

A control policy $\pi_{\text{uncert}}(\vect{x}_k)$, which may not enforce constraints, is used to achieve a task~(e.g., trajectory tracking). A safety filter, which has access to the nominal model $\bar{\vect{f}}$, takes in the input $\vect{u}_{\text{uncert}, k} = \pi_{\text{uncert}}(\vect{x}_k)$ as well as the current state of the system $\vect{x}_{k}$, and uses the nominal system to determine if the proposed input is safe~(i.e., satisfies state and input constraints, and will not lead to a constraint violation in the future). If it is not safe, it will find a safe input $\vect{u}_{\text{cert}, k}$ that minimally modifies the uncertified input. This certified input is the input applied to the system, $\vect{u}_k = \vect{u}_{\text{cert}, k}$. 

The goal of this paper is to address chattering, an undesirable behavior caused by safety filters~\cite{Koller_LMPC, zeilinger_linear}, seen in \textref{fig:drone_chat}{Fig.}. We observe chattering when the input to the system oscillates between the original, unmodified input (because it is safe) and a corrected, certified input provided by the safety filter (because the original input is unsafe). Consider a controller approaching a constraint boundary certified by a safety filter. Suppose the uncertified controller is going to violate the constraints. In that case, the safety filter will intervene, providing a safe input that may move the system in a direction opposite to the controller's intention. The controller may then reapproach the boundary in the next time steps, causing the safety filter to once again correct the input.

This oscillation deteriorates the performance and, due to the rapid input changes near the boundary, can lead to a constraint violation. This is because the safety guarantees of the safety filter depend on assumptions on the nominal model of the system and on the model error bounds, which may not be accurate for high-frequency input changes that have not been accounted for during system identification~\cite{brunke_safe_2021}. Additionally, safety filters may provide high-magnitude corrections, which can also lead to the deterioration of performance and safety.

\section{Background}
Although MPSFs can be applied to any underlying MPC framework, including output feedback MPC~\cite{lukas_ouput_mpsf}, we will use robust MPC  with state feedback as the backbone of this paper. Robust MPC is a standard method of guaranteeing safety under uncertainty.

\begin{definition}[Robust pos. control inv.~(RPCI) set~\cite{mpc_textbook}]
A set $\mathbb{P} \subseteq \mathbb{X}_{\text{c}}$ is robust positively control invariant for the system in~(\ref{eq:nominal-system}) with a controller $\pi : \RR^n \to \RR^m $ subject to state constraints $\mathbb{X}_{\text{c}}$, input constraints $\mathbb{U}_{\text{c}}$, and bounded uncertainty set $\mathbb{W}$, if the initial system state $\vect{x}_0 \in \mathbb{P}$ implies $\vect{x}_{k + 1} = \bar{\vect{f}}(\vect{x}_k, \pi(\vect{x}_k)) + \vect{w}_k \in \mathbb{P}$, $\forall \vect{w}_k \in \mathbb{W}$, $\forall k \ge 0$.
\end{definition}

Consider a robust tube-based MPC~\cite{mayne_2005, nl_mpc} performing a task on the system $\vect{f}$. Let us denote the set of possible states of the system given the known uncertainty set at time step $k$ as $\mathbb{X}_k$, and the evolution of the system as $\mathbb{X}_{k+1}  \supseteq \boldsymbol{\Phi}(\mathbb{X}_{k}, \vect{u}_{k}, \mathbb{W}) = \{\bar{\vect{f}}(\vect{x}, \vect{u}_{k}) + \vect{w}\in \RR^n \mid \forall\, \vect{x} \in \mathbb{X}_{k}, \vect{w} \in \mathbb{W}\}$. 

\begin{assumption} \label{assump:1}
    There exists a terminal set $\mathbb{X}_{\text{term}} \subset \mathbb{X}_{\text{c}}$ and a terminal controller $\pi_{\text{term}} : \RR^n \to \RR^m$ such that the following properties hold for all $\vect{x} \in \mathbb{X}_{\text{term}}$:
    \begin{align*}
        & \boldsymbol{\Phi}(\mathbb{X}_{\text{term}}, \pi_{\text{term}}(\vect{x})), \mathbb{W}) \subseteq \mathbb{X}_{\text{term}} \\
        & \pi_{\text{term}}(\vect{x}) \in \mathbb{U}_{\text{c}} \,.
    \end{align*}
\end{assumption}
\vspace*{2mm}
Thus, $\mathbb{X}_{\text{term}}$ is an RCPI set under the terminal controller $\pi_{\text{term}}$. 

Generally, MPC solves an optimization problem at each time step for an optimal input sequence over the next $H$ time steps, where $H \in \mathbb{N}$ is the horizon. At each time step $k$, MPC only applies the first input from the resulting optimal input sequence. The optimization problem for a general robust, tube-based MPC can be stated as follows:
\begin{subequations}
\begin{align}
    \min_{\vect{u}_{0|k}, ..., \vect{u}_{H-1|k}} &\sum_{j=0}^{H-1} \ell(\mathbb{X}_{j|k}, \vect{u}_{j|k}) + \ell_{\text{term}}(\mathbb{X}_{H|k}) \label{eq:nominal-mpc-objective} \\
    \text{s.t. } & \vect{x}_k \in \mathbb{X}_{0|k} \label{eq:nominal-mpc-init} \\ 
    &\mathbb{X}_{i+1|k} \supseteq \boldsymbol{\Phi}(\mathbb{X}_{i|k}, \vect{u}_{i|k}, \mathbb{W}) \\
    & \mathbb{X}_{i|k} \subseteq \mathbb{X}_{\text{c}} \,, \; \forall \; i = 0, ..., H-1\\
    & \vect{u}_{i|k} \in \mathbb{U}_{\text{c}}\\
    & \mathbb{X}_{H|k} \subset \mathbb{X}_{\text{term}} \label{eq:nominal-mpc-terminal},
\end{align}
\end{subequations}
where $\ell(\cdot)$ and $\ell_{\text{term}}(\cdot)$ are the stage and terminal cost functions, respectively, $\mathbb{X}_{i|k}$ is the set of possible states at the $(k+i)\text{-th}$ time step computed at time step $k$, and $\vect{u}_{i|k}$ is the input at the $(k+i)\text{-th}$ time step computed at time step~$k$. 

MPSFs certify a proposed input $\vect{u}_{\text{uncert}, k}$ by replacing the MPC objective function~(\ref{eq:nominal-mpc-objective}) with
\begin{align}
\label{eq:mpsf-objective-one-step}
    J_{\text{MPSF}, 1}(\vect{x}_k) = \|\pi_{\text{uncert}}(\vect{x}_k) - \vect{u}_{0|k}\|^2_{\vect{R}},
\end{align}
where $\vect{R} \in \RR^{m \times m}$ is a positive semi-definite cost matrix to weigh different input components. This objective function minimizes the difference between the next safe input $\vect{u}_{0|k}$ and the incoming uncertified input~$\vect{u}_{\text{uncert}, k}$. Therefore, the MPSF is a one-step input reference tracking MPC that tracks the input reference provided by an uncertified controller. Note that the MPSF inherits the recursive feasibility guarantees from the underlying MPC formulation with no additional assumptions required.

We will denote this standard MPSF objective function as the `one-step' objective function, as it minimizes the difference between the certified and uncertified inputs only for the upcoming time step. The remainder of the paper will deal with augmenting this objective function to improve performance and safety.

In \cite{racing} a term is added to the objective in~(\ref{eq:mpsf-objective-one-step}) to directly penalize changes in the input. This is a common MPC approach for achieving smooth state and input trajectories~\cite{mpc_textbook}. The additional penalization term is defined as follows:
\begin{align} \label{eq:j-reg}
    J_{\text{reg}, M_{\text{r}}} = \sum_{j=0}^{M_{\text{r}}-1} w_{\text{r}}(j) \|\Delta \vect{u}_{j|k}\|_{\vect{R}_{\text{r}}}^2,
\end{align}
where $w_{\text{r}}(\cdot) : \mathbb{N}_{0} \to \RR^+$ calculates the weight associated with the $j\text{-th}$ input rate, $M_{\text{r}}$ is the regularization horizon, $\Delta \vect{u}_{j|k} = \vect{u}_{j|k} - \vect{u}_{j-1|k}$, $\Delta \vect{u}_{0|k} = \vect{u}_{0|k} - \vect{u}_{0|k-1}$, and $\vect{R}_{\text{r}} \in \RR^{m \times m}$ is a cost matrix for the input rates. 

This approach directly penalizes the rate of change of the inputs, and thus it will correct the controller even when the uncertified inputs are safe, which is not the purpose of safety filters. It was used only in~\cite{racing}, which defines an MPSF specifically for controlling a remote-controlled miniature car for racing through a fixed-width track. This additional term was used due to the necessity of smooth trajectories in that specific experiment, but the term is not generally applicable. If the controller was executing high-frequency commands without violating constraints, a standard safety filter (including our proposed approach) would not interfere while the regularized approach would unnecessarily modify the control inputs to reduce the rate of change of the inputs. As the controller should not be interrupted when the constraints are inactive, such as to leverage the performance and flexibility of learning-based controllers, this regularization approach is unsuitable for most safety filtering tasks. See \textref{fig:reg_traj_sim}{Figure} for an illustration of this undesired behavior.

Another approach would be to constrain on how much the input can change, which may reflect fundamental limitations of the system or simply be used to reduce the maximum input change. However, this maximum input change would have to be carefully tuned, and would potentially limit the flexibility of the system, and the constraint would not affect the system if the input changes are smaller than this maximum value.

\section{Multi-Step Model Predictive Safety Filters}
In this section, we propose our novel objective function that leverages the nominal system dynamics and predictions of the behavior of the controller.

\subsection{Filtering Objectives with Increased Horizon Lengths}
The standard one-step MPSF objective function \eqref{eq:mpsf-objective-one-step} minimizes the difference between the current uncertified input and the immediate certified input. However, it does not reduce future corrections and thus allows chattering and jerky corrections. To reduce this, we propose additionally including upcoming corrections in the MPSF objective function:
\begin{align}
    J_{\text{MPSF}, M}(\vect{x}_k) = \sum_{j=0}^{M-1} w(j)\| \pi_{\text{uncert}}(\vect{x}_{j|k}) - \vect{u}_{j|k} \|^2_{\vect{R}}, \label{eq:mpsf-extended-objective} 
\end{align}
where $w(\cdot) : \mathbb{N}_{0} \to \RR^+$ calculates the weights associated with the $j\text{-th}$ correction and $M$ is the filtering horizon. 

In practice, the weight $w(\cdot)$ and the filtering horizon $M$ are tuning parameters. The weight $w(\cdot)$ controls the impact of future proposed inputs.
A constant $w(\cdot)$ may be used when all future proposed control inputs are equally important, while a decreasing $w(\cdot)$ may be applied when future corrections should have less impact as $j$ increases. 
The filtering horizon $M$ determines how many time steps into the future the safety filter considers when minimizing the corrections. 
Intuitively, larger $M$ values cause the safety filter to intervene early and proactively.
Note that our proposed formulation generalizes the one-step MPSF, which is recovered by setting $M=1$.

The multi-step MPSF optimization problem is then
\begin{subequations}
\label{eq:mpsf-optimization}
\begin{align}
    \min_{\vect{u}_{0|k}, ..., \vect{u}_{H-1|k}} &\sum_{j=0}^{M-1} w(j)\| \pi_{\text{uncert}}(\vect{x}_{j|k}) - \vect{u}_{j|k} \|^2_{\vect{R}} \label{eq:mpsf-objective}\\
    \text{s.t. } & \mathrm{constraints~}\eqref{eq:nominal-mpc-init}-\eqref{eq:nominal-mpc-terminal}. \label{eq:mpsf-constraints}
\end{align}
\end{subequations}
The following theorem summarizes the proposed MPSF's recursive feasibility guarantees inherited from the underlying MPC formulation.
\begin{theorem}
Consider the system to be controlled as in~(\ref{eq:true-system}) and let~(\ref{eq:nominal-system}) be satisfied with a nominal model $\bar{\vect{f}}$ and a compact uncertainty set $\mathbb{W}$. Furthermore, there exist a terminal constraint set $\mathbb{X}_{\text{term}}$ and an associated terminal controller $\pi_{\text{term}}$ according to~\autoref{assump:1}. Let the optimization in~(\ref{eq:mpsf-optimization}) be feasible at time step $k = 0$. Then the optimization in~(\ref{eq:mpsf-optimization}) is feasible for all $k \in \mathbb{N}$. 
\end{theorem}

\begin{proof}
The optimal input trajectory at time step $k=0$, $\{\vect{u}^*_{0|0}, \vect{u}^*_{1|0}, ..., \vect{u}^*_{H-1|0} \}$, is feasible by assumption. Then, the system will be within the terminal set by time step $H$, since the feasibility of the terminal constraint yields $\mathbb{X}_{H|0} \subset \mathbb{X}_{\text{term}}$. After applying $\vect{u}_{0|0}$, the system evolves to $\mathbb{X}_{1} \supseteq \boldsymbol{\Phi}(\mathbb{X}_{0}, \vect{u}_{0|0}, \mathbb{W})$. Then initializing~(\ref{eq:mpsf-optimization}) at $k = 1$ with $\mathbb{X}_{1}$, we know that the input sequence $\{\vect{u}_{1|0}, ..., \vect{u}_{H-1|0}, \pi_{\text{term}}(\vect{x}_{H | 0}) \}$ is feasible 
for all $\vect{x}_{H | 0} \in  \mathbb{X}_{H|0} \subset \mathbb{X}_{\text{term}}$. 
This yields $\mathbb{X}_{H|1} \subset \mathbb{X}_{\text{term}}$ by \textref{assump:1}{Assumption}. This can be repeated for all time steps $k \in \mathbb{N}$ and ensures there exists a feasible input trajectory that keeps the system within $\mathbb{X}_{\text{term}}$ at time step $k + H$, and thus within the constraints $\mathbb{X}_{\text{c}}$.
\end{proof}

\subsection{Approximate Black Box Control Policies}
The controller $\pi_{\text{uncert}} (\vect{x}_{k})$ may be a black box or a nonlinear controller such as a reinforcement learning control policy represented by a deep neural network. Thus, the closed-loop behavior of the controller and safety filter must be approximated in \eqref{eq:mpsf-extended-objective}. Although many approximations are possible, we found that simulating the behavior of the uncertified controller using the nominal model was the most effective approach and requires no additional learning or tuning beyond $M$ and $w(\cdot)$.

We approximate the closed-loop behavior of the uncertified controller by simulating the effect of the uncertified controller on the nominal model (see \textref{fig:general_sf_model}{Fig.}). At each time step $k$, the simulated trajectory is calculated for the next $M$ steps as follows:
\begin{subequations}
\begin{align*}
    & \bar{\vect{u}}_{\text{uncert}, j|k} = \pi_{\text{uncert}}(\bar{\vect{x}}_{j|k})\\
    &\bar{\vect{x}}_{j+1|k} = \bar{\vect{f}}(\bar{\vect{x}}_{j|k}, \bar{\vect{u}}_{\text{uncert}, j|k})\,, \forall\; j = 0, ..., M-1,
\end{align*} 
\end{subequations}
where $\bar{\vect{x}}_{j|k}$ represents the $j\text{-th}$ state of the simulated state trajectory calculated at time step~$k$ and, analogously, $\bar{\vect{u}}_{\text{uncert}, j|k}$ is the simulated input trajectory. We set $\bar{\vect{x}}_{0|k} = \vect{x}_k$. Thus, the controller is approximated as $\pi_{\text{uncert}}(\vect{x}_{k+j}) \approx \bar{\vect{u}}_{\text{uncert}, j|k}$. The weight function $w(j)$ can be chosen to reflect the accuracy of this approximation. 

This approach requires no tuning but does necessitate access to the controller for the simulation. Thus, this approach is not suitable for human teleoperation, since the controller cannot be queried, but is appropriate for black box controllers or reinforcement learning controllers, which are difficult to approximate but can be simulated. If the controller is inaccessible, other approximations of the controller can be used. Various other approaches, such as approximating the controller as a linear quadratic regulator~(LQR) or using linear regression to learn a local approximation of the control policy, were tested by the authors of this paper and found to have comparable results to the proposed approximation. However, we restrict the presentation to the simulated input trajectory due to its simplicity, lack of additional tuning parameters, and overall efficacy.

\section{Safety Filter Performance Metrics}
Previous safety filter research has focused on eliminating constraint violations of the uncertified controller. In addition, this paper focuses on minimizing the undesirable effects of safety filters such as unnecessary corrections, jerkiness, and chattering. We propose several new metrics to measure the interventions of the safety filter and the chattering. 

We consider a matrix of corrections, $\vect{C} \in \RR^{K \times m}$, where $K$ is the total number of time steps in the experiment. Each row of $\vect{C}$, $\vect{c}_k = \vect{R}^{1/2}(\vect{u}_{\text{uncert}, k} - \vect{u}_{\text{cert}, k}) \in \RR^{m}$, is the correction at time step $k=0, ..., K-1$, and $\vect{R}^{1/2}$ is the positive semi-definite square-root matrix of $\vect{R}$. Thus, the safety filter at time step $k$ minimizes the norm of the correction $\|\vect{c}_k\|_2^2$. The cost matrix~$\vect{R}$ in~(\ref{eq:mpsf-objective-one-step}) and~(\ref{eq:mpsf-extended-objective}) weighs the different components of the input. For example, it can be chosen as the identity if the inputs are normalized and of the same size. However, if~$\vect{R}$ is not identity, we account for it in the corrections by weighing them using its square-root matrix.

\subsection{Measuring Interventions} \label{sec:meas_inter}
A safety filter should minimize how significantly and how often it must intervene. We propose to measure this in two ways: the magnitude of the corrections and the number of corrections.

The magnitude of corrections is calculated as the norm of the matrix of corrections, $\|\vect{C}\|_{\text{F}}$, where $\|\cdot\|_{\text{F}}$ is the Frobenius norm. This measures how much the safety filter modified the proposed inputs in an experiment. The Frobenius norm is used as it weighs all elements of a matrix equally, extending the Euclidean norm to matrices.

The number of corrections indicates the number of time steps at which the proposed input was modified above a certain percentage tolerance $\epsilon$. This is calculated as $\sum_{k=0}^{K-1} \vect{1} \left( \frac{\|\vect{c}_k\|_2}{\|\vect{u}_{\text{cert}, k}\|_{\vect{R}}} \ge \epsilon\right)$, where $\vect{1}(\cdot)$ is the indicator function. This equation simply counts the times when the magnitude of the correction divided by the magnitude of the certified input is above the tolerance $\epsilon$. The tolerance will depend on how sensitive the system is to variations in the input.

\subsection{Measuring Chattering} \label{sec:meas_chat}
Safety filters usually increase the rate of change of the inputs through chattering and high-magnitude corrections. To measure a safety filter's impact, we propose two metrics: the maximum correction and the norm of the rate of change of the inputs.

The maximum correction is  $\max_{k\in \{0, \dots, K-1\}} \|\vect{c}_k\|_2$. Higher maximum corrections typically correspond to jerky movement. Ideally, we wish to maintain safety with small-magnitude corrections.

The rate of change of the inputs measures how much the input varied during an experiment, and is significantly increased by chattering and jerky inputs. Consider the applied inputs $\vect{u}_{k}$ for $k = 0, ..., K-1$. First, we take the numerical derivative $\delta \vect{u}_{k} = \vect{R}^{1/2}(\vect{u}_{k} - \vect{u}_{k-1})/\delta t$ for $k = 1, ..., K-1$ (where $\delta t$ is the length of a time step) and stack them into a matrix $\Delta \vect{u} = [\delta \vect{u}_{1}, ..., \delta \vect{u}_{K-1}]$. Then the norm of the rate of change of the inputs is $\|\Delta \vect{u}\|_{\text{F}}$. This metric can also be applied to an uncertified trajectory, enabling comparisons between uncertified and filtered controllers.

\section{Experimental Results}
To determine the efficacy of the proposed multi-step MPSF, we ran experiments in the safe learning-based control simulation environment  \texttt{safe-control-gym}~\cite{safe-control-gym} and on a real quadrotor, the Crazyflie 2.0 (see \textref{fig:drone_chat}{Fig.} and \textref{fig:traj_real}{Fig.}). The underlying MPC is a robust nonlinear MPC formulation~\cite{nl_mpc}, which assumes time-invariant system dynamics. Additionally, it assumes that the system is incrementally stabilizable and the model mismatch between the nominal system and the real system is bounded by a known bound. The upper bound on the model mismatch was found experimentally by comparing the true states and the states predicted by the nominal model while executing trajectories for system identification, $w_{\text{max}} = \max_{k \in \{0, \dots, K-1\}} \|\vect{x}_{k+1} - \bar{\vect{x}}_{1|k}\|_2$, and thus we define $\mathbb{W} = \{\vect{w}\in\RR^n \mid \|\vect{w}\|_2 \le w_{\text{max}}\}$. 

The experiments test the performance of the standard one-step MPSF compared to our proposed multi-step MPSF with $M=2, 5, 10$. Additionally, we consider the one-step MPSF with regularization \eqref{eq:j-reg} (with $M_{\text{r}} = 10$) to demonstrate our approach is effective even compared to an approach that directly penalizes the rate of change of the inputs along the entire trajectory, which leads to over-correcting that is unsuitable for most safety filtering tasks. For all experiments, the weight functions $w(j)$ and $w_{\text{r}}(j)$ were chosen to be $w(j) = w_{\text{r}}(j) = 0.85^{j}$ and the input channel cost matrices $\vect{R}$, and $\vect{R}_{\text{r}}$ were chosen to be the identity matrix. 

\subsection{Simulation} \label{sec:sim_results}
The MPSFs were evaluated in simulation using the \texttt{safe-control-gym}~\cite{safe-control-gym}. The experiments were done on three different systems and two different tasks: a cartpole, a two-dimensional quadrotor, and a three-dimensional quadrotor, on both stabilization and trajectory tracking tasks with box constraints on the state and input. Because of space constraints,  only the results of the cartpole trajectory tracking experiments are presented here. However, the experiments on the other systems and tasks led to similar conclusions. The remaining results can be found as part of our code release, see \url{https://tinyurl.com/mpsf-results}.

The cartpole system is a cart on a track with a pole hinged to the top of the cart. The state is $\vect{x} = [x, \dot{x}, \theta, \dot{\theta}]\T$, where $x$ is the horizontal position of the cart, $\dot{x}$ is the velocity of the cart, $\theta$ is the angle of the pole with respect to the vertical axis, and $\dot{\theta}$ is the angular velocity of the pole. The input $u\in \RR$ is the force applied to the center of mass of the cart. There is no friction. The nominal equations of motion are~\cite{safe-control-gym}:
\begin{align*}
    & \ddot{x} = \frac{u + m_p l(\dot{\theta}^2\sin\theta - \ddot{\theta}\cos\theta)}{m_c + m_p}\\
    & \ddot{\theta} = \frac{g\sin\theta + \cos\theta\left( \frac{-u - m_p l \dot{\theta}^2\sin{\theta}}{m_c + m_p} \right)}{l\left( \frac{4}{3} - \frac{m_p \cos^2\theta}{m_c + m_p}\right)},
\end{align*}
where $g$ is the gravitational acceleration, $m_c$ and $m_p$ are the masses of the cart and the pole, respectively, and $l$ is half the length of the pole. The true behavior of the cartpole in the simulation is determined not by these equations but by the PyBullet \cite{pybullet} physics engine, which these equations seek to represent. These equations were discretized using a sampling time of $\delta t = \frac{1}{15}\unit{\second}$ to get the nominal discrete-time model of the cartpole, which was found to be very accurate, with $w_{\text{max}}=0.0014$. An accurate nominal model was chosen to demonstrate that MPSFs with $M=1$ still cause chattering even when the model is accurately known. When the model is less accurately known chattering will still occur but the safety filter will be more conservative.  

The trajectory tracking task consists of tracking a sinusoidal reference in the $x$-direction with an amplitude of \SI{1}{\meter} and a period of \SI{5}{\second} while constraining the angle of the pole to be within [\SI{-9.2}{\degree}, \SI{9.2}{\degree}] and the applied force to be within [\SI{-10}{\newton}, \SI{10}{\newton}]. The chosen controllers are a linear-quadratic regulator (LQR), a proximal policy optimization (PPO) \cite{ppo} reinforcement learning controller, and a soft actor-critic (SAC) \cite{sac} reinforcement learning controller. Each MPSF was tested with an MPC constraint horizon of $H=20$ and a controller frequency of 15Hz. To find suitable and interesting starting points, random states from the constrained state space were sampled. If a starting state caused at least five constraint violations when the controllers were uncertified, but could achieve constraint satisfaction in closed-loop with the one-step MPSF, then that starting point was selected. Ten of these starting points were generated for each uncertified controller. Each MPSF was tested using these ten starting states. The starting states are the source of randomness in the results. 

\begin{figure}[tb]
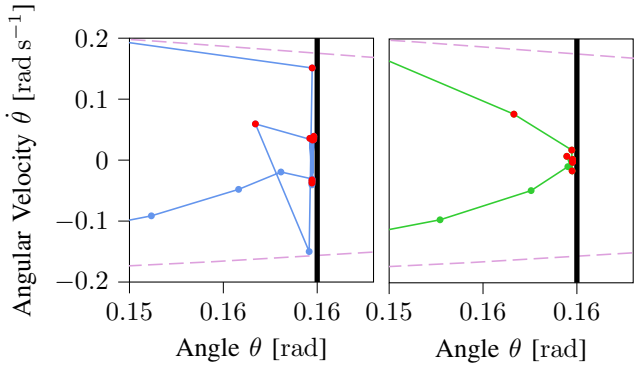

    \centering
    \vspace*{-1mm}\begin{tikzpicture}
    \node (a) at (0,0)
     {
        \input{figs/chattering_example/chatter_1s}
     };
    \node (b) at (a.east) [anchor=west,yshift=-0.3cm,xshift=-0.6cm]
     {
        \input{figs/chattering_example/chatter_pre_2}
     };
\end{tikzpicture}
    \vspace*{-6mm}
    \caption{Recorded trajectories of a cartpole trajectory tracking example when approaching a state constraint (in black) with the uncertified trajectory in purple~(dashed), the one-step MPSF in blue (left), and the proposed MPSF with $M=2$ in green (right). States at which the input was modified are indicated by a red dot. The standard one-step approach causes significant chattering, while the proposed approach greatly reduces the amplitude of the chattering.}
    \label{fig:chattering}
    \vspace*{-4mm}
\end{figure}

In the experiments, chattering is significantly reduced by our proposed multi-step MPSF. In \textref{fig:chattering}{Fig.} we see an example of chattering when certified by a one-step MPSF versus when using the proposed MPSF. In cartpole trajectory tracking, the norm of the rate of change of the inputs compared to the one-step MPSF is reduced by up to 73\% (see \textref{fig:roc_sim}{Fig.}), the magnitude of corrections are reduced by up to 25\% (see \textref{fig:magnitude_sim}{Fig.}), and the maximum correction is reduced by up to 52\% (see \textref{fig:max_sim}{Fig.}). This demonstrates that the proposed filter effectively reduces chattering, achieving a similar norm of the rate of change of the inputs compared to the uncertified control inputs (those, however, cause constraint violations), and generally decreases the overall correction effort as well. 

\begin{figure}[tb]
  \centering
  \input{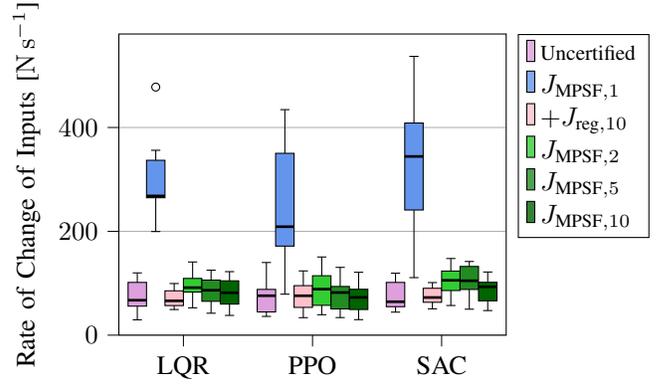}\\
  \vspace*{-1mm}
  \caption{The norm of the rate of change of the inputs (see \textref{sec:meas_chat}{Section}) for the simulated cartpole trajectory tracking experiments for the uncertified controller, the multi-step MPSFs with varying $M$ (denoted $J_{\text{MPSF}, M}$), and a one-step MPSF with regularization (denoted $+ J_{\text{reg}, 10}$). The proposed approach significantly decreases the norm of the rate of change of the inputs, up to a 73\% decrease compared to the one-step approach (in blue), without violating the constraints.}
  \label{fig:roc_sim}
  \vspace*{-3mm}
\end{figure}

\begin{figure}[tb]
  \centering
  \vspace*{2mm}
  \input{figs/simulation/cartpole_track_magnitude_of_corrections}\\
  \vspace*{-1mm}
  \caption{The magnitude of corrections (see \textref{sec:meas_inter}{Section}) for the simulated cartpole trajectory tracking experiments for the multi-step MPSFs with varying $M$ (denoted $J_{\text{MPSF}, M}$), and a one-step MPSF with regularization (denoted $+ J_{\text{reg}, 10}$). The proposed approach decreases the median magnitude of corrections by up to 25\% compared to the one-step approach, demonstrating it has an overall lower correction effort.}
  \label{fig:magnitude_sim}
  \vspace*{-4mm}
\end{figure}

\begin{figure}[tb]
  \centering
  \vspace*{2mm}
  \input{figs/simulation/cartpole_track_max_correction.tex}\\
  \vspace*{-1mm}
  \caption{The maximum correction (see \textref{sec:meas_chat}{Section}) for the simulated cartpole trajectory tracking experiments for the multi-step MPSFs with varying $M$ (denoted $J_{\text{MPSF}, M}$), and a one-step MPSF with regularization (denoted $+ J_{\text{reg}, 10}$). The proposed approach decreases the median maximum correction by up to 52\% compared to the one-step approach, demonstrating it causes less jerky corrections.}
  \label{fig:max_sim}
\end{figure}

\begin{figure}[tb]
  \centering
  \vspace*{2mm}
  \input{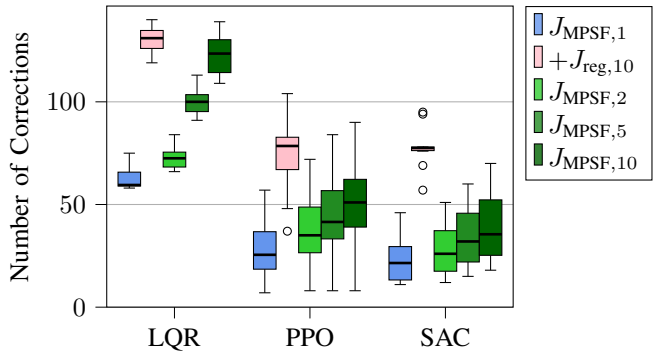}\\
  \vspace*{-1mm}
  \caption{The number of corrections (see \textref{sec:meas_inter}{Section}) for the simulated cartpole trajectory tracking experiments, using a tolerance of $\epsilon=0.1$, for the multi-step MPSFs with varying $M$ (denoted $J_{\text{MPSF}, M}$), and a one-step MPSF with regularization (denoted $+ J_{\text{reg}, 10}$). Both the proposed approach and regularization increase the number of corrections compared to the one-step MPSF (in blue), but the proposed approach decreases the median number of corrections by up to 66\% compared to the one-step MPSF with regularization.}
  \label{fig:num_sim}
  \vspace*{-3mm}
\end{figure}

The proposed filter has a comparable norm of the rate of change of inputs to the one-step MPSF with the additional regularization term but has up to a 22\% lower magnitude of corrections and 29\% lower maximum correction. Additionally, our approach does not correct the system when the constraints are inactive, reducing the number of corrections compared to the regularized approach by up to 66\% (see \textref{fig:num_sim}{Fig.}). Higher $M$ values do increase the number of corrections as the filter is intervening more proactively and early by considering corrections further in the future. However, the MPSF with $M=10$ still has fewer interventions than the regularized approach. When tracking a safe but high-frequency trajectory (see \textref{fig:reg_traj_sim}{Fig.}), the one-step and multi-step MPSFs do not correct the controller as safety is not violated. However, the regularized MPSF significantly corrects the controller, unnecessarily deviating from the desired safe trajectory. This highlights that the additional regularization term is effective close to the safety boundary reducing chattering and high-magnitude corrections; however, it limits performance away from the safety boundary. 

\begin{figure}[!tb]
  \centering
  \input{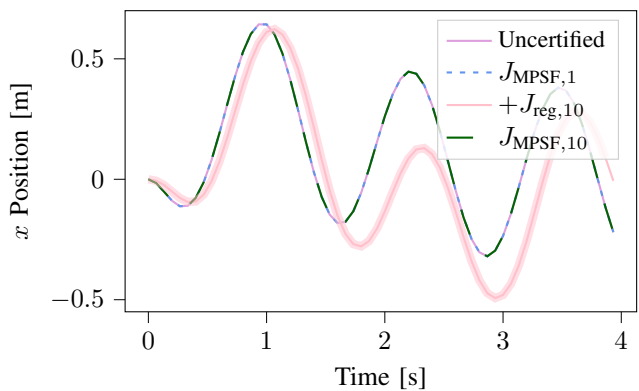}\\
  \vspace*{-1mm}
  \caption{Position trajectories of a simulated cartpole with a safe yet high-frequency sinusoidal reference trajectory when uncertified and when certified by the one-step ($J_{\text{MPSF}, 1}$), one-step with regularization ($+J_{\text{reg}, 10}$), and multi-step MPSFs ($J_{\text{MPSF}, M}$). Time steps where the safety filter executed a correction (with a tolerance $\epsilon=0.1$) are highlighted. Note that the one-step MPSF and the proposed MPSF do not correct the LQR controller as the trajectory is safe; however, the regularized approach does, causing a significant deviation from the desired trajectory.}
  \label{fig:reg_traj_sim}
  \vspace*{-3mm}
\end{figure}

To address the computational aspects of our proposed approach, we observe an increase in the time required to solve the multi-step MPSF problem compared to the one-step MPSF. The time required to solve the MPSF at one time instance increased by approximately 7\%, 22\%, and 60\% when $M=2, 5, 10$, respectively. Adding regularization increased the execution time by 45\%. 

\subsection{Robot Experiments}
The MPSFs were tested on a horizontal trajectory tracking task using a Crazyflie 2.0 quadrotor. The desired trajectory is a sinusoid in $x$ with an amplitude of \SI{1.5}{\meter} and a period of \SI{10}{\second}. The $x$ position of the drone is constrained to be within [\SI{-0.75}{\meter}, \SI{0.75}{\meter}] and the velocity of the drone is constrained to be within [\SI{-0.5}{\meter\per\second}, \SI{0.5}{\meter\per\second}] (see \textref{fig:traj_real}{Fig.}). The input $u \in \RR$ is a position setpoint the quadrotor tracks using an internal cascaded nonlinear controller~\cite{mellinger}. This commanded position setpoint is constrained to be within \SI{0.25}{\meter} of the current position to ensure smooth behavior from the internal position controller. The closed-loop dynamics in the $x$-direction can be approximated as a linear system. The linear system is identified by applying a sinusoidal reference trajectory with varying frequencies and amplitudes. We consider the state $\vect{x} = [x_1, x_2]\T \in \RR^2$, where $x_1$ is the quadrotor’s position and $x_2$ is its velocity in the $x$-direction. Using MATLAB’s System Identification Toolbox, we identified the system at 25\,Hz, as
\begin{equation*}
\vect{x}_{k+1} = \begin{bmatrix} 0.9756 & 0.0287 \\
-0.2793 & 0.8535\end{bmatrix} \vect{x}_k + \begin{bmatrix} 0.0231 \\  0.2854 \end{bmatrix} u_k + w_k, 
\end{equation*}
which we use as the nominal model in the MPSF. The model mismatch was found to be significantly higher than for the simulation experiments, with $w_{\text{max}} = 0.0449$. The chosen controller is an LQR that causes violations in the position and velocity constraints. A motion capture system provides the quadrotor position. 

The one-step, one-step with regularization (with $M_{\text{r}} = 10$), and our proposed multi-step (with $M=2, 5, 10$) MPSFs are tested five times each with an MPC horizon of $H=10$. Every test starts at the same initial position, the origin (and a static height of \SI{1}{\meter}). As seen in \textref{tab:real_experiments}{Table}, in the experiments, our proposed approach significantly reduces the norm of the rate of change of the inputs, reducing it by 80\% compared to the one-step approach when $M=10$. The maximum correction and magnitude of corrections are either maintained or decreased, and both are decreased by over 30\% compared to the one-step approach when $M=10$. We see that the one-step with regularization is outperformed in every metric by the proposed approach with $M=5$ and $M=10$, including the norm of the rate of change of the inputs. This demonstrates that our proposed approach is more effective at decreasing chattering. This is achieved by considering (or predicting) chattering in the future as opposed to directly penalizing the rate of change of the inputs, while still decreasing the maximum correction and the magnitude of corrections. 

\begin{figure}[tb]
  \centering
  \vspace*{2mm}
  \input{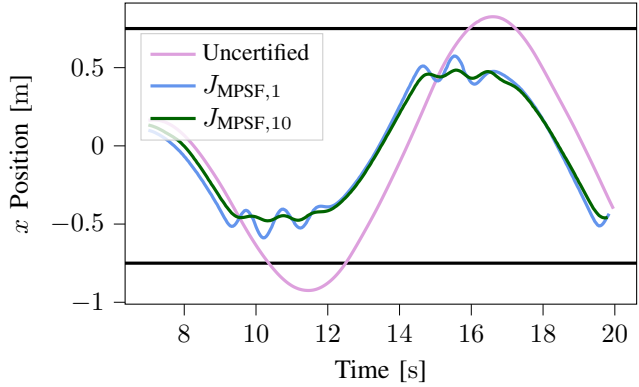}\\
  \vspace*{-1mm}
  \caption{Position trajectories of the real Crazyflie 2.0 experiment when uncertified, certified by the one-step MPSF ($J_{\text{MPSF}, 1}$), and certified by the proposed multi-step MPSF with $M=10$ ($J_{\text{MPSF}, 10}$). 
  The proposed filter reduces the peak-to-peak amplitude of chattering near the constraint boundaries~(thick black lines), reducing it from an average of \SI{16.3}{\centi\meter} to \SI{3.6}{\centi\meter}. The certified controllers do not approach the constraint boundaries closely due to the large model uncertainty and strict velocity constraints.}
  \label{fig:traj_real}
  \vspace*{-2mm}
\end{figure}

\begin{table*}[t]
	\centering
	\vspace*{1ex}
	\captionof{table}{Experimental results summary for 5 trials certifying an LQR controller flying a Crazyflie 2.0 on a sinusoidal path. The number of corrections uses a tolerance of $\epsilon=0.1$ and is out of a total of 497 time steps.}
	\begin{tabular}{lrrrrrr}
		\toprule
		Metric & Uncertified & $J_{\text{MPSF}, 1}$ & $+J_{\text{reg}, 10}$ & $J_{\text{MPSF}, 2}$ & $J_{\text{MPSF}, 5}$ & $J_{\text{MPSF}, 10}$ \\
		\midrule
		Mean norm of the rate of change of the inputs [\si{\meter\per\second}] & 2.59 & 30.26 & 15.86 & 23.21 & 13.69 & \textbf{6.18}\\
		Mean magnitude of corrections [\si{\meter}] & - & 2.96 & 2.67 & 2.89 & 2.51 & \textbf{2.03} \\
		Mean maximum correction [\si{\meter}] & - & 0.46 & 0.44 & 0.45 & 0.42 & \textbf{0.30} \\
		Mean Number of Corrections & - & \textbf{76.00} & 299.40 & 81.00 & 175.40 & 194.20 \\
		\bottomrule
	\end{tabular}
	\label{tab:real_experiments}
	\vspace*{-1mm}
\end{table*}

\section{Conclusion}
This paper proposes a modified objective function for model predictive safety filters (MPSFs) that reduces chattering while not affecting performance when away from the safety boundaries. The paper also introduces several metrics to measure the performance of safety filters. The proposed approach minimizes the input corrections of the safety filter over a multi-step horizon rather than a single time step. The future behavior of the controller is approximated by simulating its input trajectory leveraging a known nominal dynamics model. The presented approach requires little tuning and no additional assumptions compared to previous MPSFs. In simulation and real experiments, we found that the proposed multi-step MPSF greatly reduces chattering compared to the standard MPSF formulation, expanding the applicability of MPSFs to real-world systems.

\balance
\bibliographystyle{IEEEtran} 
\bibliography{bibliography}
\balance
\end{document}